\documentclass[accepted]{uai2023} 
\usepackage[british]{babel}
\usepackage{natbib} 
\usepackage{mathtools} 
\usepackage{booktabs} 
\usepackage{amsmath}
\usepackage{amsthm}
\usepackage{algorithm}
\usepackage{algpseudocode}
\usepackage{bbm}
\usepackage{newfloat}
\usepackage{listings}
\usepackage{amsfonts,amsmath}
\usepackage{bm}
\usepackage{hyperref}
\usepackage{cleveref}
\usepackage{thmtools,thm-restate}
\newtheorem{definition}{Definition}
\newtheorem{example}{Example}

\newtheorem{remark}{Remark}

\newtheorem*{proposition*}{Proposition}

\usepackage{tikz}
\usetikzlibrary{positioning,decorations.pathreplacing,shapes}
\usepackage{times}
\usepackage{url}
\usepackage{xcolor}
\usepackage{yhmath}
\usepackage{cases}
\usepackage{bm}
\usepackage{bbm}
\usepackage{dirtytalk}

\def\wfomc{\mbox{\sc wfomc}}
\def\w{\mbox{\sc w}}

\def\bh{\bm{h}}

\def\bk{{\bm{k}}}
\def\bp{{\bm{p}}}

\def\bs{{\bm{s}}}

\usepackage{mymacros}

\title{Weighted First Order Model Counting with Directed Acyclic Graph Axioms}

\author[1,2]{\href{mailto:smalhotra@fbk.eu}{Sagar Malhotra}{}}
\author[1]{Luciano Serafini}
\affil[1]{
    Fondazione Bruno Kessler, Italy
}
\affil[2]{%
    University of Trento\\
    Italy
}

\begin{document}
\maketitle

\begin{abstract}
Statistical Relational Learning (SRL) integrates First-Order Logic (FOL) and probability theory for learning and inference over relational data. Probabilistic inference and learning in many SRL models can be reduced to Weighted First Order Model Counting (WFOMC). However, WFOMC is known to be intractable ($\mathrm{\#P_1-}$ complete). Hence, logical fragments that admit polynomial time WFOMC are of significant interest. Such fragments are called \emph{domain liftable}. Recent line of works have shown the two-variable fragment of FOL, extended with counting quantifiers ($\mathrm{C^2}$) to be domain-liftable. However, many properties of real-world data can not be modelled in $\mathrm{C^2}$. In fact many ubiquitous properties of real-world data are not expressible in FOL.  Acyclicity is one such property, found in citation networks, genealogy data, temporal data e.t.c.  In this paper we aim to address this problem by investigating the domain liftability of \emph{directed acyclicity} constraints. We show that the fragment $\mathrm{C^2}$ with a Directed Acyclic Graph (DAG) axiom, i.e., a predicate in the language is axiomatized to represent a DAG, is domain-liftable. We present a method based on principle of inclusion-exclusion for WFOMC  of $\mathrm{C^2}$ formulas extended with DAG axioms. 


%
\end{abstract}


\section{Introduction}
A large part of Statistical Relational Learning (SRL) \cite{SRL_LISA,SRL_LUC} is concerned with modelling, learning and inferring over large scale relational datasets. Most SRL models assume a probability distribution over the models of a First-Order Logic (FOL) language with a finite domain. Probabilistic inference and learning in many SRL models like  Markov logic networks \cite{richardson2006markov} and Probabilistic Logic Programs \cite{PLP} can be reduced to instances of \emph{Weighted First Order Model Counting} (WFOMC) \cite{broeck2013,PTP}. WFOMC is the task of computing the weighted sum of the models of a given First Order Logic (FOL) sentence $\Phi$ over a given finite domain of size $n$. Formally, 
\begin{equation}
    \label{eq: WFOMC}
    \wfomc(\Phi,\w,n) = \sum_{\omega\models\Phi}\w(\omega)
\end{equation}
where $\w$ is a \emph{weight function} that associates a real number to each interpretation $\omega$. Fragments of FOL that admit polynomial time WFOMC w.r.t the domain cardinality are known as \emph{domain-liftable} \cite{domain_lifted_defintion}. Hence, WFOMC provides a convenient theoretical and practical tool for investigating SRL models. This generality of WFOMC applications have led to significant interest in FOL fragments that are domain liftable \cite{LIFTED_PI_KB_COMPLETION,DOMAIN_RECURSION,kazemi2016new}. A sequence of results \cite{Symmetric_Weighted, broeck2013} have led to a rather clear picture of domain-liftability of FOL formulas with at most two-variables. Especially, these results have  shown that any FOL formula in the two-variable fragment is domain-liftable \cite{Symmetric_Weighted}. However, this positive result is also accompanied by the intractability results in WFOMC, showing that  there is an FOL formula in the three variable fragment, whose (W)FOMC can not be computed in polynomial time w.r.t the domain cardinality. 

Hence, a significant effort has been made in the recent years towards expanding the domain-liftability of the two-variable fragment of FOL, with additional constraints such as \emph{functionality constraint} \cite{kuusisto2018weighted}, \emph{cardinality constraints}, \emph{counting quantifiers} \cite{kuzelka2020weighted,AAAI_Sagar} and \emph{linear order axiom} \cite{Linear_Order_Axiom}.  \\

However, many aspects of data are not captured in the domain-liftable fragments. Furthermore, crucial attributes of real-world data may not even be FOL-definable. One such example is \emph{Directed Acyclic Graphs} (DAGs):
\begin{center}
    \emph{Acyclicity is not first-order expressible.} \\
    -- \emph{Chapter 6, \cite{immerman2012descriptive}}
\end{center}
 
Furthermore, the  inexpressivity proof presented in \cite{immerman2012descriptive} works for both directed and undirected acyclicity. In this paper, we will focus on directed acyclicity or equivalently DAGs. DAGs are ubiquitous data structures, that appear in all kinds of applications. Citation networks \cite{DAG_Citation} such as CiteSeer, Cora  and PubMed, can be modeled as  DAGs. Citation networks are acyclic because a paper cannot cite itself or cite a paper that cites it. In these networks, articles are represented as nodes, and directed edges represent the citation relationships between them. Another example of directed acyclicity is  genealogy trees that trace family relationships. Genealogy trees\footnote{"Trees" is a misnomer here, as family trees are akin to DAGs where a child node has more than one paerent} can be represented as DAGs, where nodes represent individuals and directed edges represent parent-child relationships. Hence, SRL models that can express an acyclicity constraints can significantly aid learning and inference tasks in many real-world datasets.



    


In this paper, we show that WFOMC in C$^2$ expanded with a DAG axiom is domain liftable, allowing us to efficiently answer questions like:
\begin{center}
 \emph{How many DAGs with exactly (or at least or at most) $k$ sources  and exactly $m$ (or at least or at most) sinks exist ?}\\
\end{center}


Besides its application to SRL, WFOMC is also a convinient language for investigating enumerative combinatorics, as any FOL definable concept can be enumerated using WFOMC. Hence, extending WFOMC with DAG axiom can allow investigating combinatorics of many new constraints. 
\section{Background}
In this section, we will briefly introduce the notation and concepts of FOL and WFOMC. Additionally, we will review previous research on WFOMC in FO$^2$ and C$^2$, highlighting their relevance to our findings. We will also cover the necessary combinatorial concepts, such as the inclusion-exclusion principle, and derive the formula for counting DAGs.


\subsection{Basic Notation.} 
The notation $[n]$ is used to represent the set of integers $\{1,\dots,n\}$. When it is clear from the context, we use $[\overline{m}]$ to represent the set $[n] \backslash [m]$, which is the set $\{m+1, \dots, n\}$. Bold letters (such as $\bm{k}$) denote vectors, and corresponding regular font letters with an additional index (such as $k_i$) denote the components of the vectors. We represent a vector of $u$ non-negative integers as $\bm{k} = \langle k_1,...,k_u \rangle$, and use $|\bk|$ to denote the sum of its components, which is given by $\sum_{i \in [u]}k_i$. We also use the multinomial coefficient notation given by:

\begin{equation*}
\binom{|\bk|}{k_1,...,k_u} = \binom{|\bk|}{\bk}= \frac{|\bk|!}{\prod_{i\in [u]} k_i !}
\end{equation*}
For a function $f(\bm{k})$ defined on $\bm{k}$, a summation of the form $\sum_{\bm{k}}f(\bm{k})$ is taken over all possible vectors $\bm{k}$ such that $\sum_{i\in[u]}k_i = |\bk|$. When we add or subtract two vectors $\bk' = \langle k'_1,...,k'_u \rangle $ and $\bk''= \langle k''_1,...,k''_u \rangle$, we perform element-wise addition or subtraction, i.e., $\bk' + \bk'' = \langle k'_1+k''_1,...,k'_u+k''_u \rangle$.

\subsection{First Order Logic} 

First-order logic is a formal system used to express statements and reason about objects in a precise and rigorous manner. As common in WFOMC literature, we will deal with a fragment of First-Order Logic (FOL), also known as the \emph{Herbrand Logic} \cite{Herbrand_Logic}, which can be succinctly described as follows:

\begin{align*}
    \text{Herbrand Logic} :=&  \text{ First-Order Logic Syntax} \\
     &+ \text{Herbrand Semantics}
\end{align*}

Formally, we assume a function-free FOL language $\mathcal{L}$, comprising a finite set of variables $\mathcal{V}$, a finite set of relational symbols $\mathcal{R}$, and a set of constants $\Delta$ also known as the \emph{domain}. Each relational symbol $R$ in $\mathcal{R}$ has an associated arity. We use the notation $R/k$ to denote the fact that the arity of the relational symbol $R$ is $k$, where $k$ is a non-negative integer. For example, we use $P/2$ to denote the fact that $P$ is a binary relation. An \emph{atom} is an expression of the form $R(a_1, \ldots, a_k)$, where $R/k \in \mathcal{R}$, and $a_1, \ldots, a_k \in \mathcal{V} \cup \Delta$. A \emph{literal} is either an atom or its negation. A formula in FOL can be formed by combining atoms with boolean operators ($\neg, \lor$ and $\land$), using the rules of FOL syntax. Furthermore, FOL admits quantification over variables, with quantifiers of the form $\exists x_i.$ (existential quantification) and $\forall x_i.$ (universal quantification). Hence, a formula in an FOL language $\mathcal{L}$ belongs to inductive closure of  strings formed by symbols in $\mathcal{L}$, composed with boolean operators and quantifiers, using the FOL syntax rules. \emph{Free variables} of a given formula are the variables not bounded by any quantifiers. We write $\Phi(x_1,\dots, x_k)$ to denote a formula whose free variables are ${x_1, \dots, x_k}$. A \emph{sentence} is an FOL formula with no free variables and is denoted by a capital Greek letter (e.g., $\Psi$). The set of \emph{ground atoms} are atoms containing no variables. Similarly, \emph{ground literals} are literals containing no variables.  Hence, given a predicate $R/k$, and a domain $\Delta$ of size $n$, we have $n^{k}$ ground atoms of the form $R(a_1,\dots,a_k)$, where $(a_1, \dots, a_k) \in \Delta^{k}$. An interpretation $\omega$ on a finite domain $\Delta$ is a truth assignment to all the ground atoms. We say $\omega$ is a model of $\Psi$ if $\omega \models \Psi$, i.e., $\omega$ satisfies the formula $\Psi$. Given a literal (resp. ground literal) $l$, we use $pred(l)$ to denote the relational symbol in $l$. For a subset $\mathrm{\Delta'}\subset \Delta$, we use $\omega \downarrow \mathrm{\Delta'}$ to denote the partial interpretation induced by $\mathrm{\Delta'}$. Hence, $\omega \downarrow \mathrm{\Delta'}$ is an interpretation over the ground atoms containing only the domain elements in $\mathrm{\Delta'}$. We use $\omega_{R}$ to represent the partial interpretation of $\omega$ restricted to the relation $R$.
\begin{example}
    \label{ex: projection}
    Let us have a language with only two relational symbol $R$ and $B$ both of arity $2$, with a domain $\Delta = [4]$. We represent an interpretation $\omega$ as a multi-relational directed graph, where a pair of domain elements $c$ and $d$ have a red (resp. blue) directed edge from $c$ to $d$ if $R(c,d)$ (resp. $B(c,d)$) is true in $\omega$ and have no red (resp. blue) edge otherwise. Let us take for example the following interpretation $\omega$ on $[4]$:
\noindent
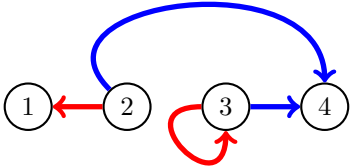
\begin{center}
\begin{tikzpicture}[node distance={13mm}, thick, main/.style = {draw, circle}] 
        \node[main] (1) {$2$}; 
        \node[main] (2) [left of=1] {$1$}; 
        \node[main] (3) [right of=1] {$3$}; 
        \node[main] (4) [right of=3] {$4$};  
        \draw[red, ->,line width = 2pt] (1) -- (2); 
        \draw[blue,->,line width = 2pt] (1) to [out=135,in=90,looseness=1.5] (4);
        \draw[red,->,line width = 2pt] (3) to [out=180,in=270,looseness=6] (3);
        \draw[blue,->,line width = 2pt] (3) -- (4); 
        \end{tikzpicture}
    \end{center}
\noindent
then $\omega' = \omega \downarrow [2]$ and $\omega'' = \omega \downarrow [\bar{2}]$ are given as:
\noindent
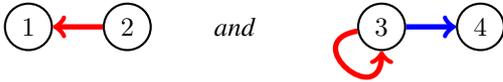
\begin{center}
    \begin{tikzpicture}[node distance={13mm}, thick, main/.style = {draw, circle}] 
            \node[main] (1) {$2$}; 
            \node[main] (2) [left of=1] {$1$}; 
            \node[main] (3) [right = 4cm of 2] {$3$}; 
            \node[main] (4) [right of=3] {$4$};  
            \draw[red,->,line width = 2pt] (1) -- (2);
            \node[fill=none,align=center]{\hskip 8em and};
            \draw[red,->,line width = 2pt] (3) to [out=190,in=270,looseness=5] (3); 
             
            \draw[blue,->,line width = 2pt] (3) -- (4); 
            \end{tikzpicture}
\end{center}
respectively.
Projecting on the predicate $R$, denoted by $\omega_{R}$ is given as:
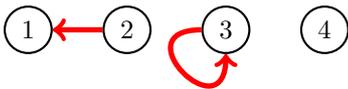
\begin{center}
    \begin{tikzpicture}[node distance={13mm}, thick, main/.style = {draw, circle}] 
            \node[main] (1) {$2$}; 
            \node[main] (2) [left of=1] {$1$}; 
            \node[main] (3) [right of=1] {$3$}; 
            \node[main] (4) [right of=3] {$4$};  
            \draw[red, ->,line width = 2pt] (1) -- (2); 
            \draw[red,->,line width = 2pt] (3) to [out=180,in=270,looseness=6] (3);
            \end{tikzpicture}
        \end{center}
Similarly, projecting on the predicate  $B$, denoted by $\omega_B$ is given as:
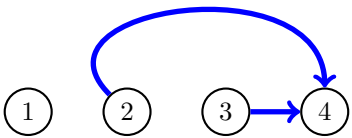
\begin{center}
    \begin{tikzpicture}[node distance={13mm}, thick, main/.style = {draw, circle}] 
            \node[main] (1) {$2$}; 
            \node[main] (2) [left of=1] {$1$}; 
            \node[main] (3) [right of=1] {$3$}; 
            \node[main] (4) [right of=3] {$4$};  
            \draw[blue,->,line width = 2pt] (1) to [out=135,in=90,looseness=1.5] (4);
            \draw[blue,->,line width = 2pt] (3) -- (4); 
            \end{tikzpicture}
        \end{center}
\end{example}

\noindent
Some key results  in this paper  (e.g. Proposition \ref{prop: mapping_1} and Proposition \ref{prop: extensions_k}) deal with expanding interpretations on two disjoint set of domain constants. Given a pair of disjoint sets of domain constants $\Delta'$ and $\Delta''$, we use $\Delta = \Delta' \uplus \Delta''$, to denote the fact that $\Delta$ is a union of two such disjoint sets.  If $\omega'$ is an interpretation on $\Delta'$ and $\omega''$ is an interpretation on $\Delta''$, then we use $\omega' \uplus \omega''$ to denote the parital interpretation on $ \Delta' \uplus \Delta''$, obtained by interpreting ground atoms over $\Delta'$ as interpreted in $\omega'$ and ground atoms over $\Delta''$ as interpreted in $\omega''$. However, the ground atoms involving domain constants from both $\Delta'$ and $\Delta''$ are left un-interpreted in $\omega' \uplus \omega''$. We further illustrate this point in the following example.

\begin{example}
    \label{ex: extension}
    Let us have an $\mathrm{FOL}$ language with only one relational symbol $R$ of arity 2. Let $
    \Delta = [3]$, and $\Delta' = [2]$ and $\Delta'' = [\bar{2}] = \{3\}$. Let us have the following two interpretations $\omega'$ and $\omega''$ on the domain $[2]$ and $[\bar{2}]$, respectively. A pair of domain elements $c$ and $d$ have a red  directed edge from $c$ to $d$ if $R(c,d)$ is true in the interpretation and have no edge otherwise. 
    
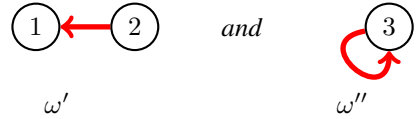
\begin{center}
        \begin{tikzpicture}[node distance={13mm}, thick, main/.style = {draw, circle}] 
                \node[main] (1) {$2$}; 
                \node[main] (2) [left of=1] {$1$}; 
                \node[main] (3) [right = 4cm of 2] {$3$}; 
                \draw[red,->,line width = 2pt] (1) -- (2);
                \node[fill=none,align=center]{\hskip 8em and};
                \draw[red,->,line width = 2pt] (3) to [out=190,in=270,looseness=5] (3); 
                 
                \end{tikzpicture}
    \end{center}
\vspace{-3em}
\begin{align*}
    \omega'  \hspace{10em}  \omega''
\end{align*}
We now create a partial interpretation $\omega' \uplus \omega''$ as follows:
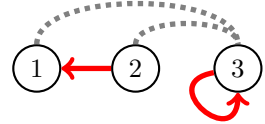
\begin{center}
\begin{tikzpicture}[node distance={13mm}, thick, main/.style = {draw, circle}] 
    \node[main] (1) {$2$}; 
    \node[main] (2) [left of=1] {$1$}; 
    \node[main] (3) [right = 2cm of 2] {$3$}; 
    \draw[red,->,line width = 2pt] (1) -- (2);
    \draw[red,->,line width = 2pt] (3) to [out=190,in=270,looseness=5] (3); 
    \draw[gray, dotted, line width = 2pt] (1) to [out=90,in=90,looseness=0.7] (3);
    \draw[gray, dotted, line width = 2pt] (2) to [out=90,in=90,looseness=0.7] (3);
     
    \end{tikzpicture}
\end{center}

The dotted lines \begin{tikzpicture}
    \draw[gray, dotted, line width = 2pt] (0,0) -- (0.5,0); 
\end{tikzpicture} represent that $R(1,3), R(3,1), R(2,3)$ and $R(3,2)$ are not interpreted in $\omega' \uplus \omega''$. Hence, a possible extension of $\omega' \uplus \omega''$ is given as follows:

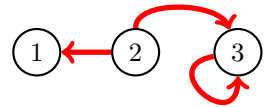
\begin{center}
    \begin{tikzpicture}[node distance={13mm}, thick, main/.style = {draw, circle}] 
        \node[main] (1) {$2$}; 
        \node[main] (2) [left of=1] {$1$}; 
        \node[main] (3) [right = 2cm of 2] {$3$}; 
        \draw[red,->,line width = 2pt] (1) -- (2);
        \draw[red,->,line width = 2pt] (3) to [out=190,in=270,looseness=5] (3); 
        \draw[red, ->, line width = 2pt] (1) to [out=90,in=110,looseness=0.8] (3);
         
        \end{tikzpicture}
    \end{center}

Where $R(2,3)$ is interpreted to be true and $R(1,3), R(3,1)$ and $R(3,2)$ are interpreted to be false. We can see that,  $\omega' \uplus \omega''$ can be extended in $2^{4}$ ways, as we have two mutually exclusive and independent choices for assigning truth values to each of the four un-interpreted .

\end{example}
\subsubsection{Counting Quantifiers and Cardinality Constraints}
The two variable fragment of FOL is largely denoted as FO$^2$. One of the most useful and a key extension to FOL is an extension with \emph{counting quantifiers} \cite{otto2017bounded}. Counting quantifiers extend the regular set of existential and universal quantifiers, with quantifiers of the form $\exists^{=k}$ (there exist exactly $k$), $\exists^{\geq k}$ (there exist at least $k$) and $\exists^{\leq k}$ (there exist at most $k$). The two-variable fragment of FOL extended with counting quantifiers is denoted by C$^2$ \cite{COUNTING_REF}. C$^2$ significantly expands the expressive power of FO$^2$, by succinctly expressing formulas with only two variables, which would otherwise require more than two variables. One such simple example is given below:

\begin{example}
    \label{ex: FOL_Counting}
    Given an $\mathrm{FOL}$ language containing only one unary predicate $P$, our goal is to write a formula which says that there are at least three distinct domain elements $c$ in the domain $\Delta$, such that $P(c)$ is true. This can be achieved with three variables as follows:
    \begin{align}
        \label{eq: three_distinct_FO2}
    \exists x. P(x) \land  \exists y. P(y) \land \exists z. P(z) \land (x \neq y) \land (y \neq z) 
    \end{align}
However, with counting quantifiers this formula can be equivalently written as:
\begin{align}
    \label{eq: three_distinct_C2}
\exists^{\leq 3} x. P(x)
\end{align}
    
\end{example}

Another interesting extension to FOL is \emph{cardinality constraints}. Cardinality constraints are constraints on the cardinality of predicates in an interpretation \cite{kuzelka2020weighted}. For example, given a formula $\Psi \land (|P| \geq 3)$, then an interpretation $\omega \models \Phi \land (|P| \geq 3)$ if and only if $\omega \models \Phi$ and the number of ground atoms in $\omega$, with predicate $P$, that are interpreted to be true, is at least $3$. Notice, that the formula $\exists^{\leq 3} x. P(x)$, as introduced in Example \ref{ex: FOL_Counting}, can be represented with a cardinality constraint $|P| \geq 3$. This relation between cardinality constraints and counting quantifiers can be generalized to WFOMC problems and is formalized in Theorem \ref{thmt@@thm: C2_WMC}  due to Ku\v{z}elka \cite{kuzelka2020weighted}.

\subsubsection{Types and Tables}

We will use the notion of $1$-types, $2$-type, and $2$-tables as presented in \cite{kuusisto2018weighted,ECML_PROJ}. Intuitively, 1-types are the set of  mutually-exclusive
 unary properties that individual domain elements can realize in a given FOL language. Formally, a $1$-type is a conjunction of maximally consistent literals containing only one variable. For example, in an FO$^2$ language with only a unary predicate $U$ and a binary predicate $R$,  ${U(x)\land R(x,x)}$ and ${U(x)\land \neg R(x,x)}$ are examples of 1-types in variable $x$. It can be seen that in this language we have only $2^{2}$ possible 1-types.  In a given interpretation $\omega$, over a set of domain constants, a single domain constant realizes one and only one 1-type. We assume an arbitrary order on the set of 1-types, hence, we use $i(x)$ to denote the $i^{th}$ 1-type. We say a domain constant $c$ realizes the $i^{th}$ 1-type in the interpretation $\omega$ if $ {\omega \models i(c)}$.  
 
2-tables can be intuitively seen as the set of mutually-exclusive
binary properties that an ordered pair of domain elements can realize in a given FOL language. Formally, a 2-table is a conjunction of maximally consistent literals containing exactly two distinct variables. Extending the previous example, ${R(x,y)\land \neg R(y,x)\land (x \neq y)}$ and ${R(x,y)\land R(y,x) \land (x \neq y )}$ are instances of 2-tables. We assume an arbitrary order on the 2-tables, hence, we use ${l(x,y)}$ to denote the $l^{th}$ 2-table.  We say an ordered pair of domain constants $(c,d)$ realizes the $l^{th}$ 2-table in an interpretation $\omega$ if $ {\omega \models l(c,d)}$.

A $2$-type is a quantifier-free formula of the form ${i(x)\land j(y) \land l(x,y) \land (x \neq y)}$ and we use ${ijl(x,y)}$ to represent it. We say an ordered pair of domain constants $(c,d)$ realizes the 2-type $ {ijl(x,y)}$ if ${\omega \models ijl(c,d)}$. We will use $u$ to denote the number of 1-types and $b$ to denote the number of 2-tables in a given FO$^2$ (or a C$^2$) language. 

\begin{definition}[1-type Cardinality Vector] An interpretation $\omega$ is said to have the 1-type cardinality vector $\bk = \langle k_1,\dots,k_u \rangle$ if and only if, for all $i \in [u]$, it has $k_i$ domain elements $c$ such that $\omega \models i(c)$, where $i(x)$ is the $i^{th}$ 1-type. If $\omega$ has 1-type cardinality vector $\bk$, then we say that $\omega \models \bk$. 
\end{definition}



It should be noted that in an interpretation $\omega$, every element in the domain realizes only one 1-type. Consequently, if a 1-type cardinality vector $\bk$ is given, then the domain cardinality is equal to $|\bk|$. Additionally, when given a fixed pair of 1-types $i$ and $j$ (where $i \neq j$) along with $\bk$, there are $k_ik_j$ pairs of domain constants $(c,d)$ such that $\omega \models i(c) \land j(d)$. Moreover, for a given 1-type $i$ and $\bk$, there exist $\binom{k_i}{2}$ pairs of domain constants $(c,d)$ that satisfy $\omega \models i(c) \land i(d)$.





\subsection{WFOMC} 
In equation \refeq{eq: WFOMC}, the $\wfomc$ function assumes that the weight function $w$ is independent of individual domain constants. This assumption implies that $w$ assigns the same weight to two interpretations that are isomorphic through the permutation of domain elements. As a result, if a domain $\Delta$ of size $n$ is given, we can use $[n]$ as the domain instead.

Moreover, this paper concentrates on a particular set of weight functions called \emph{symmetric weight functions}, which are defined as follows:

\begin{definition} (Symmetric Weight Function)
    \label{def: symm}
    Given a function-free first order logic language $\mathcal{L}$ over a domain $\Delta$, where $\mathcal{G}$ are the set of ground atoms. A symmetric weight function associates two real-valued weights  $w: \mathcal{R} \rightarrow \mathbb{R}$ and $\bar{w}: \mathcal{R} \rightarrow \mathbb{R}$ to each relational symbol in $\mathcal{L}$. The weight of an interpretation $\omega$ is then defined as:
    \begin{equation}
        \label{eq: symmweight}
        \w(\omega) = \prod_{\substack{ \omega \models g \\ g \in \mathcal{G} \\}}w(pred(g)) \prod_{\substack{ \omega \models \neg g\\ g \in \mathcal{G} }}\bar{w}(pred(g)).
    \end{equation} 
We use $(w,\bar{w})$ to denote a symmetric weight function.
\end{definition}

We will also need to invoke modularity of WFOMC-preserving reductions.
\begin{definition}[\cite{broeck2013}] A reduction $(\Phi, w,\bar{w})$ to $(\Phi', w',\bar{w}')$ is modular iff for any sentence $\Lambda$:
\begin{equation*}
    \wfomc(\Phi \land \Lambda,(w,\bar{w}),n) = \wfomc(\Phi' \land \Lambda,(w',\bar{w}'),n)  
\end{equation*}  
\end{definition}
Intuitively, modularity implies that the reduction procedure is sound under presence of other sentences $\Lambda$. And that any new sentence $\Lambda$ does not invalidate the reduction.

For the rest of the paper, whenever referring to weights, we intend symmetric weights. Hence, we will use $\wfomc(\Phi,n)$ without explicitly mentioning the weights $w$ and $\bar{w}$. Additionally, we define $\wfomc(\Phi,\bk)$ in the following manner:
\begin{align*}
    \wfomc(\Phi,\bk) := \sum_{\omega \models \Phi \land \bk} \w(\omega)
\end{align*}
where $\w(\omega)$ is the symmetric weight function.

\subsubsection{WFOMC in FO$^2$} 
Universally quantified FO$^2$ formulas are formulas of the form $\forall xy. \Phi(x,y)$, where $\Phi(x,y)$ is quantifier-free. We define $\Phi(\{x,y\})$ as $\Phi(x,x)\land \Phi(x,y)\land \Phi(y,x)\land \Phi(y,y)\land( x \neq y)$. We now define the notion of 2-type consistency with respect to a universally quantified FO$^2$ formula.

\begin{definition}[2-Type Consistency]
    \label{def: 2-type_consistency}
    Given a universally quantified $\mathrm{FO}^2$ formula $\forall xy. \Phi(x,y)$, a 2-type is consistent with $\forall xy. \Phi(x,y)$ if:
    \begin{align}
      \label{2_type_consistency}
      ijl(x,y) \models \Phi(\{x,y\})
    \end{align}  
  where the entailment in equation \ref{2_type_consistency} is checked by assuming a propositional language consisting of only the constant-free literals in the FOL language.
  \end{definition}

  \begin{example}
    \label{ex:truth-assignment}
The following is an example of a consistent 2-type for the formula $\forall xy. \Phi(x,y):= \forall xy. A(x) \land R(x,y) \rightarrow A(y)$:
  \begin{align*}
   &\tau(x,y) := \\
   &\neg A(x) \land R(x,x) \land \neg A(y) \land R(y,y) \land \neg R(x,y) \land R(y,x)
  \end{align*}
  It is easy to see that, assuming a propositional language consisting of constant-free literals, i.e., with propositional variables  $\{A(x),A(y),R(x,x),R(y,y),R(x,y),R(y,x)\}$, that:
  $$\tau(x,y) \models \Phi(\{x,y\})$$
  Hence, $\tau(x,y)$ is a consistent 2-type for  $\forall xy. \Phi(x,y)$.
  \end{example}


To analyze the domain-liftability of universally quantified FO$^2$ formulas, a crucial observation is that a pair of domain constants $(c,d)$ in an interpretation $\omega \models \forall xy. \Phi(x,y)$ can only realize the 2-type $ijl(c,d)$ if the 2-type is consistent with the formula $\forall xy. \Phi(x,y)$, meaning that $ijl(x,y) \models \Phi({x,y})$. We have formulated this idea in the following proposition.

\begin{restatable}{prop}{prop: ijl_FO}
    \label{prop: ext}
    Given a universally quantified FO$^2$ formula, $\forall xy. \Phi(x,y)$ interpreted over a domain $\Delta$. Then $\omega \models \forall xy. \Phi(x,y)$ iff  for any pair of distinct domain constants $(c,d)$ such that $\omega \models ijl(c,d)$, we have that $ijl(x,y)$ is consistent with $\Phi(\{x,y\})$, i.e. $ijl(x,y) \models \Phi(\{x,y\})$.
\end{restatable}

\begin{proof} If $\omega \models \forall xy. \Phi(x,y)$ and $\omega \models ijl(c,d)$, then $\omega \models \forall xy. \Phi(x,y) \land ijl(c,d)$. Now, $ijl(c,d)$ is a complete truth assignment to the ground atoms containing only the domain constants $c$ or $d$ or both. Hence, $ \omega \models \forall xy. \Phi(x,y) \land ijl(c,d)$ only if $ijl(c,d) \models \Phi(\{c,d\})$ i.e. only if ${ijl(x,y) \models \Phi(\{x,y\})}$. 

If in $\omega$ all pair of domain constants realize only the 2-types $ijl(x,y)$ consistent with $\forall xy. \Phi(x,y)$. Then $\omega \models \Phi(\{c,d\})$, for all pair of domain constants $(c,d)$. Hence, $\omega \models \forall xy. \Phi(x,y)$. 
    
\end{proof}

To facilitate the treatment of WFOMC, we will now introduce some weight parameters associated with an FO$^2$ language. Specifically, consider an FO$^2$ language $\mathcal{L}$ with symmetric weight functions $(w,\bar{w})$, and let $\mathcal{I}$ denote the set of atoms in $\mathcal{L}$ that contain only variables and are not grounded. We will then define two weight parameters for each 1-type $i(x)$ and 2-table $l(x,y)$.

$$w_i = \prod_{\substack{i(x) \models  g \\ g \in \mathcal{I}}}w(pred(g)) \prod_{\substack{ i(x) \models \neg g \\ g \in \mathcal{I}}} \bar{w}(pred(g))$$
and

$$v_l =  \prod_{\substack{l(x,y) \models  g \\ g \in \mathcal{I}}}w(pred(g)) \prod_{\substack{ l(x,y) \models \neg g \\ g \in \mathcal{I}}} \bar{w}(pred(g))$$

In the following, we will present a combinatorial formula for WFOMC in the universally quantified fragment of FO$^2$. Our presentation is based on the treatment proposed in \cite{Symmetric_Weighted}.

\begin{restatable}[\cite{Symmetric_Weighted}]{thm}{thm: Symm_WMC}
    \label{thm: beam}
    Given a universally quantified FO$^2$ formula $\forall xy. \Phi(x,y)$, interpreted over a domain $[n]$, then the weighted model count of the models $\omega$ such that $\omega \models \forall xy.\Phi(x,y)$ and $\omega$ has the 1-type cardinality $\bk$ is given as:
    \begin{align}
        \label{eq: WFOMC_Beame}
        \wfomc(\forall xy.\Phi(x,y),\bk) =  \binom{n}{\bm{k}} \prod_{i\in [u]}w_i^{k_i}
        \prod_{\substack{{i\leq j \in[u]}}}\!\!\! r_{i j}^{\bk(i,j)} 
    \end{align} 
where $\bk(i,j)$ is defined as follows:
\begin{align*}
    \label{kij}
  \bk(i,j) =
  \begin{cases} 
      \frac{k_{i}(k_{i} - 1)}{2} & \text{if $i=j$} \\
       k_{i}k_{j} & \text{otherwise} \\
     \end{cases}
\end{align*} 
where we define $r_{ij} = \sum_{l\in[b]}n_{ijl}v_{l}$, where $n_{ijl}$ is $1$ if  $ijl(x,y) \models \Phi(\{x,y\})$ and $0$ otherwise.

\end{restatable}

 

\begin{proof}

    Consider a 1-type cardinality vector $\bk$, where $k_i$ represents the number of constants that realize the 1-type $i$. Since, 1-types are realized mutually-exclusively by domain constants, i.e., no domain constant can realize two 1-types in the same interpretation, there are $\binom{n}{\bk}$ ways of assigning 1-types to $n = |\bk|$ domain constants. Suppose that a domain constant $c$ realizes the $i^{th}$ 1-type, then $c$ contributes to the weight of the interpretation $\omega \models \bk$ multiplicatively with the weight $w_i$. Therefore, for a given $\bk$, the contribution due to 1-type realizations is given by $\prod_{i\in [u]}w_i^{k_i}$. Now consider an interpretation $\omega$ and a pair of domain constants $c$ and $d$, such that $\omega \models i(c)\land j(d)$. Using Proposition \ref{prop: ext}, we know that $(c,d)$ can realize the 2-table $l(c,d)$ only if $ijl(x,y) \models \Phi({x,y})$. Therefore, in an arbitrary interpretation $\omega$, the multiplicative weight contribution due to a pair of constants $(c,d)$ realizing the $l^{th}$ 2-table, such that $\omega \models i(c) \land j(d)$, is given by $n_{ijl}v_{l}$. Also, each ordered pair of constants can realize exactly one and only one 2-table. Hence, the sum of the weights of the possible 2-table realization of a pair of domain constants $(c,d)$ such that $i(c)$ and $j(d)$ is given as $r_{ij}=\sum_{l}n_{ijl}v_{l}$. Furthermore, given 1-type assignments $i(c)$ and $j(d)$, the ordered pair $(c,d)$ can realize 2-table independently of all other domain constants.  Finally, There are $\bk(i,j)$ possible such pairs, contributing a weight $$\prod_{\substack{{i\leq j \in[u]}}}\!\!\! r_{i j}^{\bk(i,j)}$$ 
\end{proof}

Clearly, equation \eqref{eq: WFOMC_Beame} can be computed in polynomial time w.r.t domain cardinality. Furthermore, there are only polynomially many $\bk$, in the size of the domain. Hence, $\wfomc(\forall xy.\Phi(x,y),n)$ given as:
$$\sum_{|\bk|=n}\wfomc(\forall xy.\Phi(x,y),\bk)$$ 
can be computed in polynomial time w.r.t domain size $n$. 

\cite{broeck2013} show that any FOL formula with existential quantification can be modularly reduced to a WFOMC preserving universally quantified FO$^2$ formula, with additional new predicates and negative weights. Hence, showing that FO$^2$ is domain-liftable.
  
\subsubsection{WFOMC in C$^2$}
Ku\v{z}elka \cite{kuzelka2020weighted} showed that it is possible to reduce the problem of WFOMC in C$^2$ to a problem of WFOMC in FO$^2$ with cardinality constraints, and this reduction is independent of the cardinality of the domain. In order to prove the domain liftability of FO$^2$ with cardinality constraint,  Ku\v{z}elka uses Lagrange interpolation based arguments. In the following, we present a simplified version of the proof presented by Ku\v{z}elka.

\begin{restatable}[\cite{kuzelka2020weighted}, slightly reformulated]{thm}{thm: C_WMC}
    \label{th: cardinality}
    
    Let $\Phi$ be a first-order logic sentence. Let $\Gamma$ be a arbitrary cardinality constraint. Then $\wfomc(\Phi \land \Gamma,\bk)$ can be computed in polynomial time with respect to the domain cardinality, relative to the $\wfomc(\Phi,\bk)$ oracle. 
\end{restatable}
\begin{proof} Let us consider an FOL language $\mathcal{L}$ that contains $r$ relational symbols denoted as $\{R_i/a_i\}_{i \in [r]}$. Let $\omega$ be an interpretation and let $\bm{\mu} = \langle |R_1|, \dots, |R_{r}|\rangle$ be the vector comprising the cardinality of each predicate $R_i$ in $\omega$. The evaluation of $\w(\omega)$ can be easily carried out utilizing the definition of symmetric weight functions (Definition \ref{def: symm}). Moreover, we can see that any two interpretations that have the same predicate cardinalities $\bm{\mu}$ as $\omega$ possess the same weight $\w(\omega)$. Therefore, we use $\w_{\bm{\mu}}$ to indicate the weight $\w(\omega)$.

Given an FOL formula $\Phi$, let $A_{\bm{\mu}}$ be the number of models $\omega \models \Phi \land \bm{\mu}$. Clearly, the following holds:
\begin{equation}
    \label{eq: cardinality}
    \wfomc(\Phi,\bk) = \sum_{\bm{\mu}}A_{\mu}\w_{\bm{\mu}}
\end{equation}

For each predicate $R_i/a_i$ in the FOL language $\mathcal{L}$, there exist $n^{a_i}$ ground atoms. Therefore, there are $n^{\sum_{i\in[r]}a_i}$ potential values of $\bm{\mu}$. Hence, there are polynomially many $\bm{\mu}$ vectors with respect to $n$. By evaluating $\wfomc(\Phi,\bk)$ for $n^{\sum_{i\in[r]}a_i}$ distinct weight function pairs $(w,\bar{w})$, we obtain a linear system of $n^{\sum_{i\in[r]}a_i}$ equations consisting of $n^{\sum_{i\in[r]}a_i}$ variables $A_{\bm{\mu}}$. This system can be solved using Gauss-elimination algorithm in $O(n^{3\sum_{i\in[r]}a_i})$ time. Once we obtain all $A_{\bm{\mu}}$, we can compute the value of any cardinality constraint as follows:
\begin{equation}
    \label{eq: cardinality_constraint}
    \wfomc(\Phi \land \Gamma,\bk) = \sum_{\bm{\mu \models \Gamma}}A_{\mu}\w_{\bm{\mu}}
\end{equation}

where $\bm{\mu \models \Gamma}$ represents the fact that the predicate cardinalities $\bm{\mu}$, satisfy the cardinality constraint $\Gamma$. Since, there are a only polynomial number of $\bm{\mu}$ vectors, equation \eqref{eq: cardinality_constraint} can be computed in polynomial time. 
\end{proof}

\begin{remark} In equation \eqref{eq: cardinality_constraint}, we assume that $\bm{\mu} \models \Gamma$ can be checked in polynomial time wrt $n$. Which is a reasonable assumption for all our purposes.  
\end{remark}

\begin{remark} 
    \label{rem: FOL_inexpressible_Card}
    In the proof presented above (and in \cite{kuzelka2020weighted}), the first-order definability of $\Phi$ is never invoked. This property has also been exploited for imposing cardinality constraints with tree axiom in \cite{Tree}.    
\end{remark}

Theorem \refeq{th: cardinality} extends domain-liftability of any sentence $\Phi$ to its domain liftability with cardinality constraints. We now move onto the results on domain-liftability of C$^2$.

\begin{restatable}[\cite{kuzelka2020weighted}]{thm}{thm: C2_WMC}
    \label{kuzelka_C2}
     The fragment of first-order logic with two variables and counting quantifiers is domain-liftable.
\end{restatable}

The central idea behind Theorem \ref{kuzelka_C2} is that the problem of WFOMC for a C$^2$ sentence $\Phi$ can be transformed into a problem of WFOMC for an FO$^2$ sentence $\Phi'$ on an extended vocabulary with additional cardinality constraints $\Gamma$. This vocabulary includes additional weighted predicates that are assigned a weight of either $1$ or $-1$. For a more detailed explanation of Theorem \ref{kuzelka_C2}, please refer to \cite{kuzelka2020weighted} and \cite{AAAI_Sagar}. It is important to note that this transformation is modular, the modularity of this transformation has been utilized   to demonstrate the domain-liftability of C$^2$ extended with Tree axiom \cite{Tree} and Linear Order axiom \cite{Linear_Order_Axiom}. We will also exploit this modularity to get Theorem 
\ref{thm: counting_acyclic} of this paper, i.e., for extending domain-liftability of DAG constraints to C$^2$.

\subsection{Principle of Inclusion-Exclusion}
Given a set of finite sets $\{A_i\}_{i\in [n]}$, let $A_{J}:= \bigcap_{j\in J}A_j$ for an arbitrary subset $J$ of $[n]$. Then the principle of inclusion-exclusion (PIE) states that:
\begin{equation}
    \label{P_IE}
    \Big|\bigcup_{i}A_i\Big| = \sum_{\emptyset \neq J \subseteq [n]} (-1)^{|J|+1} \big| A_{J} \big|
\end{equation}
For all subsets $J,J' \subseteq [n]$, such that $|J| = |J'| = m$ for some $m \geq 1$, if $A_{J}$ and $A_{J'}$ have the same cardinality, then there are $\binom{n}{m}$ terms in equation \eqref{P_IE}, with value $A_{[m]}$. Hence, equation \eqref{P_IE} reduces to:  
\begin{equation}
    \label{P_IE_Symm}
    \Big|\bigcup_{i}A_i\Big| = \sum_{m=1}^{n} (-1)^{m+1} \binom{n}{m} A_{[m]}
\end{equation}
\begin{remark}
    PIE can be easily extended to the case when $A_i$ are sets of weighted FOL interpretations, where each interpretation $\omega$ has a weight $\w(\omega)$, where $\w$ is the symmetric weight function as given in Definition \ref{def: symm}. In this case PIE allows us to computed the weighted sum of all the interpretations in $\bigcup_{i}A_i$. 

    Let $\w(A_i)$ denote the weighted sum of all the interpretations in $A_i$, Then the PIE reduces to:
    \begin{equation}
        \label{P_IE_w}
        \w\Big( \bigcup_{i}A_i \Big) = \sum_{\emptyset \neq J \subseteq [n]} (-1)^{|J|+1} \w(A_{J})
    \end{equation}
Similarly, when $\w(A_{J})$ and $\w(A_{J'})$ are the same for each $m = |J| = |J'|$, we have that:
\begin{equation}
    \label{P_IE_Symm_w}
    \w\Big( \bigcup_{i}A_i \Big) = \sum_{m=1}^{n} (-1)^{m+1} \binom{n}{m} \w(A_{[m]})
\end{equation}
\end{remark}
\subsection{Counting Directed Acyclic Graphs}
\label{sec: DAG_Count}
A \emph{Directed Acyclic Graph (DAG)} is a directed graph such that  starting from an arbitrary node $i$ and traversing an arbitrary path along directed edges, we would never arrive at node $i$. We now present the derivation of a recursive formula for counting the number of DAGs.

Let the nodes be the set $[n]$ and let $A_i$ be the set of DAGs on $[n]$ where node $i$ has indegree zero. Since every DAG has at least one node with in-degree zero, we have that the total number of DAGs i.e. $a_n$ is given as $|\bigcup_{i\in [n]}A_i|$. The number of DAGs such that all nodes in $J \subseteq [n]$ have in-degree zero is then given as $A_{J}:= \bigcap_{j \in J} A_j$. Let us assume that $J = [m]$ for some $1 \leq m\leq n$. We now derive a method for computing $A_{[m]}$. We make the following three observations for deriving the formula for counting the DAGs in $A_{[m]}$.

\begin{itemize}
    \item Observation 1. If $\omega \in A_{[m]}$, then there are no edges between the nodes in $[m]$, as otherwise a node in $[m]$ will have a non-zero in-degree. In other words, only directed edges from $[m]$ to $[\bar{m}]$ are allowed.
    \item Observation 2. If $\omega \in A_{[m]}$, then subgraph of $\omega$ restricted to $[\bar{m}]$ i.e. $\omega \downarrow [\bar{m}]$ is a DAG. And the subgraph of $\omega$ restricted to $[m]$ is just an empty graph, i.e., the set of isolated nodes $[m]$ with no edges between them.
    \item Observation 3. Given a DAG on $[\bar{m}]$, then it can be extended to $2^{m(n-m)}$ DAGs in $A_{[m]}$. This is because DAGs in $A_{[m]}$ have no edges between the nodes in $[m]$. They only have outgoing edges from $[m]$ to $[\bar{m}]$. For extending a given DAG on $[\bar{m}]$ to a DAG in $A_{[m]}$, we can either draw an out-going edge from $[m]$ to $[\bar{m}]$ or not. Giving us two choices for each pair of nodes in $[m]\times [\bar{m}]$. Hence, there are $2^{|[m]\times [\bar{m}]|} = 2^{m(n-m)}$ ways to extend a given DAG on $[\bar{m}]$ to a DAG in $A_{[m]}$.   
\end{itemize}

The number of possible DAGs on $[\bar{m}]$ is $a_{n-m}$. Due to Observation 3, we have that $A_{[m]}$ has $2^{m(n-m)} a_{n-m}$ DAGs obtained by extending the DAGs on $[m]$. Furthermore, due to Observation 1 and Observation 2, these are all the possible DAGs in $A_{[m]}$. Hence, $|A_{[m]}| =  2^{m(n-m)} a_{n-m}$. Now, we can repeat this argument for any $m$ sized subset of $[n]$. Hence, if $|J|=|J'| = m$ then $A_{J} = A_{J'} =  2^{m(n-m)} a_{n-m}$. Hence, using the principle of inclusion-exclusion as given in equation \eqref{P_IE_Symm}, we have that:

\begin{equation}
\label{eq: Count_DAG}    
a_n = \sum_{m=1}^{n}(-1)^{m+1}\binom{n}{m}2^{m(n-m)}a_{n-m}
\end{equation}
Notice that replacing $n-m$ with $l$ in equation \eqref{eq: Count_DAG}, it can be equivalently written as:
\begin{equation}
    \label{eq: Count_DAG_l}    
    a_n = \sum_{l=0 }^{n-1}(-1)^{n-l+1}\binom{n}{l}2^{l(n-l)}a_{l}
\end{equation}
This change of variable allows us to write a bottom-up algorithm for counting DAGs, as given in Algorithm \ref{alg:algorithm_DAG}. Based on this algorithm we now show that counting DAGs can be performed in polynomial time with respect to the number of nodes $n$.

\begin{restatable}{prop}{prop: poly_DAG}
    The number of labelled DAGs over $n$ nodes can be computed in polynomial time.
\end{restatable}
\begin{proof}
    We define $a_0=1$ by convention and then by using equation \eqref{eq: Count_DAG_l} in Algorithm \ref*{alg:algorithm_DAG}, we incrementally compute $a_1, a_2 ...$, saving each result in a list given by $A$. The for loop runs in time $O(n)$, and in each run in line 5, we perform other $O(n)$ operations. Hence, the algorithm runs in $O(n^2)$. 
\end{proof}

\begin{algorithm}[tb]
    \caption{Number of DAG on $n$ nodes}
    \label{alg:algorithm_DAG}
    \begin{algorithmic}[1]
    \State \textbf{Input}: $n$
    \State \textbf{Output}: $a_n$
        \State $A[0] \gets 1$ 
        \For{$i=1$ to $n$}
        \State $A[i] \gets \sum_{l=0}^{i-1}(-1)^{i-l+1}\binom{i}{l}2^{l(i-l)}A[l]$ 
        \EndFor\\
        \textbf{return} $A[n]$
        \end{algorithmic}
\end{algorithm}


\section{WFOMC with DAG Axiom}
In this section we extend the approach used for counting DAGs in equation \eqref{eq: Count_DAG} to WFOMC of FO$^2$ and C$^2$ formulas with a DAG Axiom. First, we formally define the DAG axiom. We then present Proposition \ref{prop: mapping_1}, Proposition \ref{prop: extensions_k} and Proposition \ref{prop: atleast_weighted}, analogous to Observation 1, 2 and 3 respectively, as presented in the subsection \ref{sec: DAG_Count}. We then  use principle of inclusion-exclusion to compute the WFOMC of universally quantified FO$^2$ formulas extended with a DAG axiom. And show our method to be domain-liftable. The proposed apporach is then extended to admit full FO$^2$, Cardinality constraints and C$^2$. We finally extend the DAG axiom, with additional unary predicates that represent sources and sinks of the DAG.

\begin{definition}
    \label{def: Acyclic graph}
    Let $\Phi$ be a first-order logic sentence, possibly containing the binary relation $R$. An interpretation $\omega$ is a model of $\Psi = \Phi \land Acyclic(R)$ if and only if:
    \begin{itemize}
        \item $\omega$ is a model of $\Phi$, and
        \item $\omega_R$ forms a Directed Acyclic Graph
    \end{itemize}

\end{definition}

\begin{definition}
    \label{def: atleast_m}
    Let $\Psi = \Phi \land Acyclic(R)$, where $\Phi$ is a first-order logic sentence, be interpreted over the domain $[n]$. Let $1 \leq m \leq n$. Then $\omega$ is a model of $\Psi_{[m]}$ if and only if $\omega$ is a model of \/ $\Psi$ on $[n]$ and the domain elements in $[m]$ have zero $R$-indegree. 
\end{definition}
Notice that due to Definition \ref{def: atleast_m}, for the domain $[n]$, $\Psi_{[n]}$ is equivalent to $\Psi' = \Phi \land \neg R(x,y)$.  


\begin{restatable}{prop}{Propmapping}
    \label{prop: mapping_1}
 Let $\Psi = \forall xy. \Phi(x,y) \land Acyclic(R)$ and $\Psi' = \forall xy. \Phi(x,y) \land \neg R(x,y)$, where $\Phi(x,y)$ is quantifier-free, be  interpreted over $[n]$. Let $1 \leq m \leq n$. If $\omega$ is a model of $\Psi_{[m]}$,  then $\omega \downarrow [m]  \models \Psi'$ and $\omega \downarrow [\bar{m}]  \models \Psi $. 
\end{restatable}

\begin{proof}[Proof Sketch] The proposition is a consequence of following three facts: (1) Since, $\forall xy.\Phi(x,y)$ is an FO$^2$ formula, then if $\omega \models \forall xy.\Phi(x,y)$, then $\omega \downarrow [m] \models  \forall xy.\Phi(x,y)$ and $\omega \downarrow [\bar{m}] \models  \forall xy.\Phi(x,y)$; (2) In $\omega \models \Psi_{[m]}$, $\omega_{R}$ cannot have an edge in $[m]$. Hence, $\omega \downarrow [m] \models \Psi'$; and (3) Subgraph of a DAG is a DAG, hence if $\omega \models \Psi_{[m]}$, then  $\omega \downarrow [\bar{m}] \models \Psi_{[m]}$. We provide the detailed proof in Appendix.
\end{proof}

\begin{restatable}{prop}{prop: extensions_k }
    \label{prop: extensions_k}
     Let $\Psi = \forall xy. \Phi(x,y) \land Acyclic(R)$ and $\Psi' = \forall xy. \Phi(x,y) \land \neg R(x,y)$, where $\Phi(x,y)$ is quantifier-free, be interpreted over the domain $[n]$. Let $\omega'$ be a model of \/ $\Psi'$  on the domain $[m]$ and let $\omega''$ be a model of $\Psi$ on the domain $[\bar{m}]$. Then the number of extensions $\omega$, of $\omega' \uplus \omega''$, such that $\omega \models \Psi_{[m]} \land \bk$ is given as:
\begin{equation}
    \label{eq: extensions}
    \prod_{i,j\in[u]} n_{ij}^{k'_i \cdot k''_j} 
\end{equation}
where $k'_i$ and $k''_i$ are the number of domain constants realizing the $i^{th}$  1-type in $\omega'$ and $\omega''$ respectively. We define $n_{ijl}$ to be $1$ if  $ijl(x,y) \models \Phi(\{x,y\}) \land \neg R(y,x)$ and $0$ otherwise and $n_{ij} = \sum_{l \in [b]}n_{ijl}$. 
\end{restatable}

\begin{proof}
In order to obtain an interpretation $\omega \models  \Psi_{[m]} \land \bk $ on the domain $[n]$ from $\omega' \uplus \omega''$, we only need to extend $\omega' \uplus \omega''$ with interpretations of the ground-atoms containing $(c,d) \in [m] \times [\bar{m}]$. For a given pair $(c,d) \in [m] \times [\bar{m}]$, let $\omega' \models i(c)$ and $\omega'' \models j(d)$. Since $\omega$ is a model of $\forall xy. \Phi(x,y)$, we must have that  $ijl(c,d) \models \Phi(\{c,d\})$. Furthermore, since we want that every domain element in $[m]$ has indegree zero, we cannot have $R(d,c)$. Hence,  we must have that $ijl(c,d) \models \Phi(\{c,d\}) \land \neg R(d,c)$. Hence, the number of 2-tables that can be realized by $(c,d)$ is given by $n_{ij}$. Since there are $k'_i$ domain elements $c$ realizing the $i^{th}$ 1-type in $\omega'$ and $k^{''}_j$ domain elements $d$ realizing the $j^{th}$ 1-type in $\omega''$, the number of extensions $\omega$, of $\omega' \uplus \omega''$, such that $\omega \models \Psi_{[m]} \land \bk$ is given by expression \eqref{eq: extensions}.
\end{proof}

\begin{restatable}{prop}{prop_atleast_weighted}
    \label{prop: atleast_weighted}
        Let $\Psi = \forall xy. \Phi(x,y) \land Acyclic(R)$ and $\Psi' = \forall xy. \Phi(x,y) \land \neg R(x,y)$, where $\Phi(x,y)$ is quantifier-free. Then: 
    \begin{equation}
        \begin{split}
        \label{eq: k_k'_m}
        &\wfomc(\Psi_{[m]},\bk) =\\
          &\sum_{\substack{\bk= \bk'+\bk''\\ |\bk'|=m}}  \prod_{i,j\in[u]} r_{ij}^{k'_ik''_j} \cdot \wfomc(\Psi',\bk')\cdot\wfomc(\Psi,\bk'' )
        \end{split}
    \end{equation}
        where $\bk'+\bk''$ represents the element-wise sum of integer-vectors $\bk'$ and $\bk''$, such that $|\bk'|=m$ and $|\bk''| = |\bk|-m$. Also, $r_{ij} = \sum_{l}n_{ijl}v_{l}$, where $n_{ijl}$ is $1$ if  ${ijl(x,y) \models \Phi(\{x,y\}) \land \neg R(y,x)}$ and $0$ otherwise. 
            
    \end{restatable}

\begin{proof} The WFOMC of $\Psi'$ on $[m]$, with 1-type cardinality vector $\bk'$ is given as $\wfomc(\Psi',\bk')$. Similarly, the WFOMC of $\Psi$ on $[\bar{m}]$, with 1-type cardinality vector $\bk''$ is given as $\wfomc(\Psi,\bk'')$. Due to proposition \ref{prop: extensions_k}, each pair of models counted in $\wfomc(\Psi',\bk')$ and $\wfomc(\Psi,\bk'')$, can be extended in  $\prod_{i,j\in[u]} n_{ij}^{k'_ik''_j}$ ways to a model of $\Psi_{[m]} \land \bk$. It is easy to see that the total multiplicative weight contribution of these extensions is given as $\prod_{i,j\in[u]} r_{ij}^{k'_ik''_j}$. The summation in \eqref{eq: k_k'_m} runs over all possible  realizable 1-type cardinalities over $[m]$ and $[\bar{m}]$, represented by $\bk'$ and $\bk''$ respectively, such that they are consistent with $\bk$, i.e. when $\bk = \bk'+ \bk''$. Hence, formula \eqref{eq: k_k'_m} gives us the WFOMC of the models $\omega$, such that $\omega \downarrow [m]  \models \Psi'$, $\omega \downarrow [\bar{m}]  \models \Psi$ and $\omega \models \Psi \land \bk$ where the domain constant in $[m]$ have zero $R$ indegree. Due to proposition \ref{prop: mapping_1}, we have that these are all the models such that $\omega \models \Psi \land \bk$ and the domain constants in $[m]$ have zero $R$ indegree.    
\end{proof}



\begin{restatable}{prop}{prop:FO_DAG_Acyclic} The first order model count of the formula $\Psi = \forall xy. \Phi(x,y) \land Acyclic(R)$, where $\Phi(x,y)$ is quantifier-free, is given as:
    \begin{equation}
        \label{eq: DAG_Acyclic}
        \wfomc(\Psi ,\bk) = \sum_{m=1}^{|\bk|}(-1)^{m + 1}\binom{|\bk|}{m}\wfomc(\Psi_{[m]},\bk)
    \end{equation} 
    
\end{restatable}

\begin{proof} The proof idea is very similar to the case for counting DAGs as given in \eqref{eq: Count_DAG}. Let the domain be $[n]$, hence $|\bk| = n$. Let $A_{i}$ be the set of models $\omega\models \Psi$, such that $\omega$ has 1-type cardinality $\bk$ and the domain element $i$ has zero $R$-indegree. Since, each DAG has atleast one node with zero $R$-indegree, our goal is to compute  $\w(\cup_{i\in [n]}A_i)$. Let $J\subseteq [n]$ be an arbitrary set of domain constants. Let $A_{J}= \bigcap_{j\in J}A_j$ for an arbitrary subset $J$ of $[n]$. Then using principle of inclusion-exclusion as given in equation \eqref{P_IE_Symm_w}, we have that:
    \begin{equation}
        \label{e: PIE_Acyclic}
        \wfomc(\Psi ,\bk)  = \sum_{\emptyset \neq J \subseteq [n]} (-1)^{|J|+1} \w(A_{[m]}) 
    \end{equation}   
Now, $A_{[m]}$ is the set of models such that domain elements in $[m]$ have zero $R$-indegree. Hence, $A_{[m]}$ are exactly the models of $\Psi_{[m]}$. Furthermore, notice that in Proposition \ref{prop: mapping_1}, Proposition \ref{prop: extensions_k} and Proposition \ref{prop: atleast_weighted}, $[m]$ can be replaced with any $m$-sized subset $J$ of $[n]$. Hence, for all $ J\subseteq [n]$, such that $|J|=m$, we have that $\w(A_{J}) =   \wfomc(\Psi_{[m]},\bk)$. Hence, equation \eqref{e: PIE_Acyclic} reduces to equation \eqref{eq: DAG_Acyclic}.
\end{proof}
We make a change of variable in equation \eqref{eq: DAG_Acyclic} (similar to equation \eqref{eq: Count_DAG_l}), by replacing $m$ with $|\bk|-l$, to obtain the following equation:

\begin{align}
    \label{eq: DAG_Acyclic_FO_l}
    \begin{split}
        &\wfomc(\Psi ,\bk)= \\
    &\sum_{l=0}^{|\bk|-1}(-1)^{ |\bk| -l + 1}\binom{|\bk|}{l}\wfomc(\Psi_{[|\bk|-l]},\bk)
    \end{split}  
\end{align} 
We provide  pseudocode for evaluating equation \eqref{eq: DAG_Acyclic_FO_l} in Algorithm \ref{alg:algorithm_FO_DAG}, namely WFOMC-DAG. We now  analyse how WFOMC-DAG works and show that it runs in polynomial time with respect to domain cardinality $|\bk| = n$.  

WFOMC-DAG takes as input $\Psi = \forall xy. \Phi(x,y) \land Acyclic(R)$ and $\bk$ -- where $\Phi(x,y)$ is a quantifier-free formula and $\bk$ is a 1-type cardinality vector, such that $|\bk| = n$ -- and returns $\wfomc(\Psi,\bk)$. In line $3$, an array $A$ with $u$ indices is initiated and $A[\mathbf{0}]$ is assigned the value  $1$, where $\mathbf{0}$ corresponds to the $u$ dimensional zero vector. The for loop in line $5-7$ incrimentally computes $\wfomc(\Psi,\bp)$, where the loop runs over all $u$-dimensional integer vectors $\bp$, such that $p_i \leq k_i$, in lexicographical order. The number of possible $\bp$ vectors is atmost $n^{u}$. Hence, the  for loop in line 5 runs at most $n^{u}$ iterations. In line 6, we compute $\wfomc(\Psi, \bp)$ as given in equation \eqref{eq: DAG_Acyclic_FO_l}. Also in line 6,  the function $\overline{\wfomc}(\Psi_{[m]},\bp)$ --- that computes $\wfomc(\Psi_{[m]},\bp)$ ---is called at most $|\bp|-1$ times, which is bounded above by $n$. $A[\bp]$ stores the value $\wfomc(\Psi,\bp)$. Hence, as $\bp$ increments in lexicographical order, $A[\bp]$, stores the value of $\wfomc(\Psi,\bp)$. In the function $\overline{\wfomc}(\Psi_{[m]},\bs)$, the number of iterations in the for loop is bounded above by $n^{2u}$. And $\wfomc(\Psi',\bs')$ is an FO$^2$ WFOMC problem, again computable in polynomial time. Hence, the algorithm WFOMC-DAG runs in polynomial time w.r.t domain cardinality. Notice that since loop 5-7 runs in lexicographical order, the $A[\bs'']$ required in the function $\overline{\wfomc}(\Psi_{[m]},\bs)$ are always already stored in $A$. Now, there are only polynomially many $\bk$ w.r.t domain cardinality. Hence, computing $\wfomc(\Psi,\bk)$ over all possible $\bk$ values, we can compute $\wfomc(\Psi, n)$ in polynomial time w.r.t domain cardinality. Furthermore, using the modular WFOMC preserving skolemization process as provided in \cite{broeck2013}, we can easily extend this result to the entire FO$^2$ fragment. Hence, leading to the following theorem:


\begin{algorithm}[tb]
    \caption{WFOMC-DAG}
    \label{alg:algorithm_FO_DAG}
    \begin{algorithmic}[1]
    \State \textbf{Input}: $\Psi, \bk$
    \State \textbf{Output}: $\wfomc(\Psi,\bk)$
        \State $A[\mathbf{0}] \gets 1$ \Comment{$A$ has $u$ indices}
        \State \Comment{$\mathbf{0} = \langle 0,...,0 \rangle$ }
        \For{$\mathbf{0} < \bp \leq \bk$ where $ \bp \in \mathbb{N}_{0}^{u}$} \Comment{Lexical order} 
        \State ${\!A[\bp] \gets \!\sum_{l=0}^{|\bp|-1}(-1)^{|\bp|-l+1}\binom{|\bp|}{l}{\overline{\wfomc}}(\!\Psi_{[|\bp|-l]},\bp)}$ 
        \EndFor\\
        \textbf{return} $A[\bk]$ 
        \Function{$\overline{\wfomc}$}{$\Psi_{[m]}$, $\bs$} \Comment{Equation \eqref{eq: k_k'_m}}
        \State $S = 0$
        \For{$\bs' + \bs'' = \bs$ and $|\bs'| = m$}
        \State $S \gets S + \prod_{i,j\in[u]} r_{ij}^{s'_is''_j} \cdot \wfomc(\Psi',\bs')\cdot A[\bs'']$ 
        \EndFor
        \State \Return $S$
        \EndFunction
        \end{algorithmic}
\end{algorithm}

\begin{restatable}{thm}{thm: FO_DAG_Lifted } Let $\Psi = \Phi \land Acyclic(R)$, where $\Phi$ is an FO$^2$ formula. Then $\wfomc(\Psi,n)$ can be computed in polynomial time with respect to the domain cardinality.
\end{restatable}

Using Theorem \ref{th: cardinality} and Remark \ref{rem: FOL_inexpressible_Card}, we can also extend domain-liftability of FO$^2$, with DAG axiom and cardinality constraints.

\begin{restatable}{thm}{thm: C_DAG_Lifted }Let $\Psi = \Phi \land Acyclic(R)$, where $\Phi$ is an FO$^2$ formula, potentially also containing cardinality constraints. Then $\wfomc(\Psi,n)$ can be computed in polynomial time with respect to the domain cardinality.
\end{restatable}

Furthermore, since WFOMC of any C$^2$ formula can be modularly reduced to WFOMC of an FO$^2$ formula with cardinality constraints \cite{kuzelka2020weighted}. We also have the following theorem:

\begin{restatable}{thm}{thm: Counting}
    \label{thm: counting_acyclic}
    Let $\Psi = \Phi \land Acyclic(R)$, where $\Phi$ is an C$^2$ formula. Then $\wfomc(\Psi,n)$ can be computed in polynomial time with respect to the domain cardinality.
\end{restatable}

\subsection{Source and Sink}
\begin{definition}
    \label{def: Acyclic graph}
    Let $\Phi$ be a first order sentence, possibly containing some binary relation $R$, a unary relation $Source$ and a unary relation $Sink$. Then a structure $\omega$ is a model of  $\Psi = \Phi \land Acyclic(R, Source, Sink)$ if and only if:
    \begin{itemize}
        \item $\omega$ is a model of $\Phi \land Acyclic(R)$, and
        \item In the DAG represented by $\omega_{R}$, the sources of the DAG are interpreted to be true in $\omega_{Source}$.
        \item In the DAG represented by $\omega_{R}$, the sinks of the DAG are interpreted to be true in $\omega_{Sink}$. 
    \end{itemize}

\end{definition}
The $Source$ and the $Sink$ predicate can allow encodicng constraints like $\exists^{=k}x.Source(x)$ or $\exists^{=k}x.Sink(x)$.

\begin{restatable}{thm}{Source_Sink}Let $\Psi = \Phi \land Acyclic(R,Source,Sink)$, where $\Phi$ is a C$^2$ formula. Then $\wfomc(\Psi,n)$ can be computed in polynomial time with respect to the domain cardinality.
\end{restatable}
\begin{proof}
  The sentence  $\Psi$ can be equivalently written as: 
  \begin{align}
    \begin{split}
    &\Phi \land Acyclic(R) \\
    & \land \forall x. Source(x) \leftrightarrow \neg \exists y. R(y,x) \\
    & \land \forall x. Sink(x) \leftrightarrow \neg \exists y. R(x,y)\\
\end{split}
  \end{align}
which is a FO$^2$ sentence extended with DAG constraint.
\end{proof}

\section*{Conclusion}
In this paper we demonstrate the domain liftability of FO$^2$ and C$^2$ extended with a Directed Acyclic Graph Axiom. We then extend our results with Source and Sink predicates, which can allow additional constraints on the number of sources and sinks in a DAG. These results can potentially allow better modelling of datasets that naturally appear with a DAG structure \cite{Citeseer}. In future, we aim at investigating successor, predecessor and ancestory constraints in FOL extended with DAG axioms.

\begin{acknowledgements}
  We would like to thank Andrea Micheli for the fruitful and interesting discussion time.
\end{acknowledgements}
\newpage
\section*{Appendix}
\Propmapping*
\begin{proof} We have that $\omega \models \Psi$. Hence, we have that: 
     \begin{align*}
        \omega &\models \bigwedge_{(c,d) \in [n]^{2}} \Phi(c,d) \\
        \Rightarrow \omega &\models \bigwedge_{(c,d) \in [m]^{2}} \Phi(c,d) \bigwedge_{(c,d) \in [\bar{m}]^{2}} \Phi(c,d) \\
        &\bigwedge_{(c,d) \in [\bar{m}]\times [m]} \Phi(c,d) \bigwedge_{(c,d) \in [m] \times [\bar{m}]} \Phi(c,d) 
    \end{align*}
Since, $\omega \models \bigwedge_{(c,d) \in [m]^{2}} \Phi(c,d)$ and $\omega \models \bigwedge_{(c,d) \in [\bar{m}]^{2}} \Phi(c,d)$, we have that  $\omega \downarrow [m]  \models \forall xy. \Phi(x,y)$ and $\omega \downarrow [\bar{m}]  \models \forall xy.\Phi(x,y)$. Now, since $[m]$ has zero $R$-indegree, it can only have outgoing $R$-edges to $[\bar{m}]$. Hence, we can infer that $\omega \downarrow [m]  \models \forall xy. \neg R(x,y)$. Now, $\omega_{R}$ is a DAG, then so is $\omega_{R}\downarrow [\bar{m}]$. Hence, $\omega\downarrow [\bar{m}] \models Acyclic(R)$. Hence, $\omega \downarrow [m]  \models \Psi'$ and $\omega \downarrow [\bar{m}]  \models \Psi$.
\end{proof}

\newpage

\bibliography{main}

\end{document}